\documentclass[conference]{IEEEtran}
\IEEEoverridecommandlockouts
\usepackage{cite}
\usepackage{amsmath,amssymb,amsfonts,amsthm}
\usepackage{algorithmic}
\usepackage{algorithm}
\usepackage{graphicx}
\usepackage{textcomp}
\usepackage{xcolor}
\usepackage{booktabs}
\usepackage{enumitem}

\usepackage{color}  
\usepackage{hyperref}
\hypersetup{
    colorlinks=true, 
    linktoc=all,     
    linkcolor=blue,  
}

\newtheorem{assumption}{Assumption}
\newtheorem{theorem}{Theorem}

\newtheorem{lemma}{Lemma}

\newtheorem{problem}{Problem}

\def\BibTeX{{\rm B\kern-.05em{\sc i\kern-.025em b}\kern-.08em
    T\kern-.1667em\lower.7ex\hbox{E}\kern-.125emX}}
\begin{document}

\title{An Exploration-free Method for a Linear Stochastic Bandit Driven by a Linear Gaussian Dynamical System
\\
\thanks{}
}

\author{Jonathan Gornet \IEEEmembership{Student Member, IEEE}, Yilin Mo \IEEEmembership{Member, IEEE}, and Bruno Sinopoli \IEEEmembership{Fellow, IEEE}
\thanks{J. Gornet and B. Sinopoli are with the Department of Electrical and Systems Engineering, Washington University in St. Louis, St. Louis, MO 63130, USA (email: jonathan.gornet@wustl.edu; bsinopoli@wustl.edu).}
\thanks{Yilin Mo is with the Department of Automation, Tsinghua University,
        Beijing, China 100084 (email: ylmo@tsinghua.edu.cn). }%
}



\maketitle

\begin{abstract}

In stochastic multi-armed bandits, a major problem the learner faces is the trade-off between exploration and exploitation. Recently, exploration-free methods---methods that commit to the action predicted to return the highest reward---have been studied from the perspective of linear bandits. In this paper, we introduce a linear bandit setting where the reward is the output of a linear Gaussian dynamical system. Motivated by a problem encountered in hyperparameter optimization for reinforcement learning, where the number of actions is much higher than the number of training iterations, we propose Kalman filter Observability Dependent Exploration (KODE), an exploration-free method that utilizes the Kalman filter predictions to select actions. Our major contribution of this work is our analysis of the performance of the proposed method, which is dependent on the observability properties of the underlying linear Gaussian dynamical system. We evaluate KODE via two different metrics: regret, which is the cumulative expected difference between the highest possible reward and the reward sampled by KODE, and action alignment, which measures how closely KODE's chosen action aligns with the linear Gaussian dynamical system's state variable. To provide intuition on the performance, we prove that KODE implicitly encourages the learner to explore actions depending on the observability of the linear Gaussian dynamical system. This method is compared to several well-known stochastic multi-armed bandit algorithms to validate our theoretical results. 

\end{abstract}


\section{Introduction}

Decision-making under uncertainty is an important real-world problem.
Previous work that has rigorously focused on this problem includes the Stochastic Multi-Armed Bandit (SMAB) problem, which consists of the interaction between a learner and an environment \cite{lattimore2020bandit}. For each interaction, called a round, the learner selects an action and in response the environment samples the reward from a distribution dependent on the chosen action. The learner's goal is to maximize the accumulated reward for a horizon length, which is the total number of interactions between the learner and the environment. The metric for measuring performance of a proposed method in SMAB is regret, which is the cumulative expected difference between the highest possible reward and the sampled reward from the method's chosen action for each round. 

The SMAB commonly encounters the problem of \textit{exploration}---gathering information about the actions---versus \textit{exploitation}---committing to actions that are predicted to return the highest reward. A well-known principle for addressing this problem is the principle of optimism in the face of uncertainty---the learner assumes that the reward output of each action is the highest plausible predicted reward given a set confidence level---which was implemented in the Upper Confidence Bound (UCB) algorithm in \cite{agrawal1995sample}. More recently, results in SMAB have focused on exploration-free methods, or methods that refrain from exploring all the actions explicitly. These methods are motivated by situations where the number of actions is greater than the number of rounds in \cite{bayati2020unreasonable}. Additionally, in some linear bandit environments---where the reward is the inner product of an unknown parameter vector and a known chosen action vector---changes in the environment can induce exploration even in exploration-free methods. For example, an adversary changing the unknown parameter vector each round promotes the exploration-free method to sample uncertain actions, as shown in \cite{kannan2018smoothed}. Another example is when the unknown parameter vector stochastically changes each round, driving exploration, which is analyzed in \cite{bastani2021mostly}. This has led to the development of methods that adaptively choose when to explore such as \cite{pmlr-v108-hao20b}. The work above highlights the need for analyzing exploration-free methods in different SMAB environments. 

Inspired by a classical problem of hyperparameter optimization in machine learning (see HyperBand in \cite{li2017hyperband}), in this paper we propose and analyze an exploration-free method for a linear bandit environment where the reward is the output of a known Linear Gaussian Dynamical System (LGDS). In other words, the reward is the inner product of an action vector and the LGDS state variable. By proposing this environment, we envision that the results can be used for hyperparameter optimization during training of reinforcement learning neural networks as the authors \cite{parker2020provably}---inspired by theoretical results in \cite{bogunovic2016time}---have modeled the problem similarly to a LGDS. In this environment, the number of actions---where an action is a hyperparameter configuration---is much larger than the number of training iterations, making it intractable to explore all the actions. 

The paper's contributions are as follows. We propose a linear bandit SMAB problem where the reward is the output of a known LGDS---inner product of an action vector and the LGDS state variable. We assume that the actions, which are a set of vectors, have a $\ell_2$ norm of 1. Our main contribution is the analysis of our proposed exploration-free method, called the Kalman filter Observability Dependent Exploration (KODE), which chooses actions that align most closely with the Kalman filter state prediction. KODE's performance is first analyzed via two different metrics. The first analysis is the performance with respect to regret, which is the cumulative expected difference between the optimal reward and the sampled reward based on the learner's chosen action. The following analysis focuses on how closely the learner's action aligns with the optimal action at each round, which is dependent on the LGDS parameters. Following the performance analysis, we discuss how KODE has an implicit exploration term that perturbs the learner to explore depending on the observability of the LGDS. We provide numerical results validating KODE's performance based on our theoretical predictions. 

The layout of the paper is as follows. For Section \ref{sec:Problem Formulation}, the proposed environment is formulated. For Section \ref{sec:Myopic}, we introduce KODE. Subsection \ref{sec:performance} will provide performance analyses of KODE and Subsection \ref{subsec:greedy_performance} will discuss how properties of the LGDS promote exploration in KODE. Numerical results will be presented in Section \ref{sec:numerical_results}. The paper is concluded with discussions of future directions in Section \ref{sec:Conclusion}. 

\subsection{Related Work}

In this paper, the proposed environment is based on the linear bandit. In linear bandits, recent studies have developed algorithms for when the unknown parameter vector changes such as \cite{cheung2019learning,zhao2020simple,kuroki2024best}. In \cite{cheung2019learning}, the authors address the case when the changes in the unknown parameter vector are bounded by proposing a sliding-window algorithm that uses 
the samples collected within a window for estimating the unknown parameter vector. For \cite{zhao2020simple}, the environment changes slightly where now the norm of the unknown parameter vector is bounded. The algorithm proposed in \cite{zhao2020simple} restarts after a set number of rounds. Finally, for \cite{kuroki2024best}, the authors tackle the adversarial and stochastic case of the linear bandit, where the adversarial case is when the unknown parameter vector is chosen by an adversary prior to each learner's interaction. In this work, the authors used a Follow-the-Regularized Leader approach, which is when the learner chooses actions based on what was the optimal action the previous round with a regularizer \cite{lattimore2020bandit}. 

A variation of the SMAB relevant to this paper is the non-stationary SMAB problem, where the distribution of the reward for each action can change each round. The authors in \cite{besbes2014stochastic} developed an algorithm that considers the case when the changes in the reward distribution are subjected to a budget. In \cite{slivkins2008adapting}, the authors analyze a number of bandit algorithms applied to an environment where the rewards are generated by Brownian motion. The results are extended to an autoregressive $s$-step process in \cite{chen2023non}.

\smallskip
\noindent\textbf{Notation:} For any $x\in\mathbb{R}^n$ and $y\in\mathbb{R}^n$, we have the inner product $\left\langle x, y\right\rangle = x^\top y \in \mathbb{R}$. 

\section{Problem Formulation}\label{sec:Problem Formulation}

In this paper, we will consider the linear bandit with an unknown parameter vector driven by a known LGDS. For background, the linear bandit has the following expression for the reward $X_t \in \mathbb{R}$:
\begin{equation}
    X_t = \left\langle a_t, z \right\rangle + \eta_t \nonumber,
\end{equation}
where $a_t \in \mathcal{A} \subseteq \mathbb{R}^d$ is the action, $z \in \mathbb{R}^d$ is an unknown parameter vector, and $\eta_t \in \mathbb{R}$ is zero-mean noise. For our proposed environment, we will set the unknown parameter vector $z$ to be $z_t$ where $z_t$ dynamically changes according to the state of a LGDS. This leads to our proposed linear bandit environment which is formulated below. 
\begin{equation}\label{eq:known_LGDS}
    \begin{cases}
        z_{t+1} & = \Gamma z_t + \xi_t, ~ z_0 \sim \mathcal{N}\left(0,\Sigma_0\right)\\
        X_t & = \left\langle a_t, z_t \right\rangle + \eta_t
    \end{cases}, 
\end{equation}

In the LGDS, $z_t \in \mathbb{R}^d$ is the LGDS's state variable which is unknown and unobserved by the learner. The reward $X_t$ is observed by the learner and is the inner product of an action vector $a_t \in \mathcal{A}$ and the LGDS state variable $z_t$. The process noise $\xi_t \in \mathbb{R}^d$ and measurement noise $\eta_t \in \mathbb{R}$ are independent Gaussian noises where $\xi_t \sim \mathcal{N}\left(\mathbf{0},Q\right), Q \succeq 0$ and $\eta_t \sim \mathcal{N}\left(0,\sigma^2\right), \sigma > 0$. The measurement noise $\eta_t$ is sampled after the action $a_t \in \mathcal{A}$ is chosen. It is assumed that the learner knows the LGDS state matrix $\Gamma \in \mathbb{R}^{d \times d}$, action space $\mathcal{A}$, and the statistics of the noise terms $\xi_t \in \mathbb{R}^d$ and $\eta_t\in \mathbb{R}$---learner knows $Q$ and $\sigma^2$. We denote a learner as the entity that is interacting with the LGDS \eqref{eq:known_LGDS}. For the LGDS \eqref{eq:known_LGDS}, we pose the following assumptions: 
\begin{assumption}\label{assumption:unit_sphere}
    The action $\mathcal{A} \subseteq \mathbb{R}^d$ is constrained to the unit sphere $\mathbb{S}^{d-1}$, i.e. 
    \begin{equation}
        \mathcal{A} \subseteq \mathbb{S}^{d-1} \triangleq \left\{a \mid \left\Vert a \right\Vert_2 = 1, a \in \mathbb{R}^d\right\} \nonumber. 
    \end{equation}

    The set $\mathcal{A}$ also contains only $k > 0$ actions, where $k$ is a known parameter. 
\end{assumption}

Assumption \ref{assumption:unit_sphere} is posed to focus on the observability of the LGDS \eqref{eq:known_LGDS}. The LGDS \eqref{eq:known_LGDS} is defined to be observable for rounds $t_0$ to $t_1$ if the \textit{Observability Gramian} matrix $\mathcal{O}\left(\Gamma,t_0,t_1\right)$ is positive definite \cite{chen1995linear}, where the \textit{Observability Gramian} is defined to be 
\begin{equation}\label{eq:observability_gramian}
    \mathcal{O}\left(\Gamma,t_0,t_1\right) \triangleq \sum_{\tau=t_0}^{t_1} \left(\Gamma^\top\right)^\tau a_\tau a_\tau^\top \Gamma^\tau \in \mathbb{R}^{d \times d}. 
\end{equation}

Based on the definition of observability with respect to the \textit{Observability Gramian}, observability is invariant to the magnitude of the action vectors $a \in \mathcal{A}$, unless the magnitude is 0. 

\begin{assumption}\label{assumption:controllability}
    The matrix pair $\left(\Gamma,Q^{1/2}\right)$ is controllable. Controllability implies that the following matrix is rank $d$ \cite{chen1995linear}: 
    \begin{equation}
        \begin{pmatrix}
            Q^{1/2} & \Gamma Q^{1/2} & \dots & \Gamma^{d-1} Q^{1/2}
        \end{pmatrix} \nonumber. 
    \end{equation}
\end{assumption}

Assumption \ref{assumption:controllability} is required to ensure the existence of the Kalman filter, where the Kalman filter is the optimal one-step predictor of the state variable $z_t$. Keep in mind that Assumption \ref{assumption:controllability} implies every element in LGDS's state vector $z_t$ is excited by the process noise $\xi_t$. 

The following problem is posed for addressing the SMAB environment with a known LGDS \eqref{eq:known_LGDS}:
\begin{problem}\label{problem_formulation}
    Let the following LGDS \eqref{eq:known_LGDS} be posed with Assumptions \ref{assumption:unit_sphere} and \ref{assumption:controllability}. Design a policy $\pi_t$ such that the learner minimizes regret $R_n$, which is defined to be 
    \begin{equation}\label{eq:regret_definition}
        R_n \triangleq \sum_{t=1}^n \mathbb{E}_{\pi_t}\left[X_t^* - X_t\right], 
    \end{equation}
    where $n > 0$ is an unknown parameter, $X_t^*$ is the highest possible reward at round $t$, and $X_t$ is the sampled reward based on the policy $\pi_t$'s chosen action at round $t$. 
\end{problem}

A key issue for addressing Problem \ref{problem_formulation} is that the LGDS state variable $z_t$ is unknown; therefore, the future reward for each action $a \in \mathcal{A}$ is unknown. We propose to predict the state $z_t$ to predict the future reward for each action $a \in \mathcal{A}$. To accomplish this, we will use the optimal one-step predictor (in the mean-squared error sense) of the LGDS state $z_t$, the Kalman filter. The Kalman filter can be expressed in the one-step predictor form below:
\begin{equation}\label{eq:Kalman_Filter}
    \begin{cases}
        \hat{z}_{t+1|t} & = \Gamma \hat{z}_{t|t-1} + \Gamma K_t \left(X_t - \left\langle a_t,\hat{z}_{t|t-1}\right\rangle \right) \\
        P_{t+1|t} & = g\left(P_{t|t-1},a_t\right) \\
        K_t & = P_{t|t-1} a_t \left(a_t^\top P_{t|t-1}a_t + \sigma^2\right)^{-1}  \\
        \hat{X}_{t|t-1} & = \left\langle a_t, \hat{z}_{t|t-1}\right\rangle
    \end{cases}. 
\end{equation}

The variable $\hat{z}_{t|t-1} \triangleq \mathbb{E}\left[z_t \mid \mathcal{F}_{t-1}\right]$ and $\mathcal{F}_{t-1}$ is the the sigma algebra of previous rewards $X_0,\dots,X_{t-1}$. The error covariance matrix $P_{t|t-1}\triangleq \mathbb{E}\left[\left(z_t - \hat{z}_{t|t-1}\right)\left(z_t - \hat{z}_{t|t-1}\right)^\top \mid \mathcal{F}_{t-1}\right]$ is the output of the difference Riccati equation $g\left(P_{t|t-1},a_t\right)$ which is defined below 
\begin{multline}\label{eq:riccati_equation}
    g\left(P_{t|t-1},a_t\right) \triangleq \Gamma P_{t|t-1}\Gamma^\top + Q \\ - \Gamma P_{t|t-1} a_t\left(a_t^\top P_{t|t-1} a_t + \sigma^2\right)^{-1} a_t^\top P_{t|t-1} \Gamma^\top. 
\end{multline}

The following lemma provides known facts about the Kalman filter \cite{sinopoli2005optimal}:

\begin{lemma}\label{lemma:Kalman_Facts}
    The following facts are true for the Kalman filter \eqref{eq:Kalman_Filter}:
    \begin{itemize}
        \item $\mathbb{E}\left[e_{t|t-1}^\top \hat{z}_{t|t-1}\mid \mathcal{F}_{t-1}\right] = 0$. 
        \item $\mathbb{E}\left[z_t^\top S z_t\mid \mathcal{F}_{t-1}\right] = \hat{z}_{t|t-1}^\top S \hat{z}_{t|t-1} + \mbox{tr}\left(S P_{t|t-1}\right)$ for all $S \succeq 0$.
        \item $\mathbb{E}\left[\mathbb{E}\left[z_t \mid \mathcal{F}_t\right]\mid \mathcal{F}_{t-1}\right] = \mathbb{E}\left[z_t \mid \mathcal{F}_{t-1}\right]$. 
    \end{itemize}
\end{lemma}

In this paper, we will approach Problem \ref{problem_formulation} by introducing a method that chooses actions $a_t \in \mathcal{A}$ that align most closely with the Kalman filter \eqref{eq:Kalman_Filter} state prediction $\hat{z}_{t+1|t}$. We will then analyze this method's performance. 

\section{An Exploration-Free Method}\label{sec:Myopic}

In this section, we propose Algorithm \ref{alg:greedy_method}, Kalman filter Observability Dependent Exploration (KODE), for solving Problem \ref{problem_formulation}. In the algorithm, action $a \in \mathcal{A}$ is chosen based on the following optimization problem:
\begin{equation}\label{eq:Greedy_Method}
    a_t = \underset{a \in \mathcal{A}}{\arg\max} \left\langle a, \hat{z}_{t|t-1} \right\rangle.
\end{equation}

\begin{algorithm}[!t]
\caption{Kalman filter Observability Dependent Exploration (KODE)}\label{alg:greedy_method}
 \begin{algorithmic}[1]
\STATE \textbf{Input}: $\Gamma$, $\mathcal{A}$, $Q$, $\sigma$, $P_{0|-1}$, $\hat{z}_{0|-1}$
\FOR{$t=1,2,\dots,n$}
    \STATE $a_t = \underset{a \in \mathcal{A}}{\arg\max} \left\langle a, \hat{z}_{t|t-1} \right\rangle$  \COMMENT{Select action} 
    \STATE Observe $X_t = \left\langle a_t, z_t \right\rangle + \eta_t$ 
    \STATE Update $\hat{z}_{t+1|t}$ and $P_{t+1|t}$ in the Kalman filter \eqref{eq:Kalman_Filter}
\ENDFOR
\end{algorithmic}
\end{algorithm}

In the optimization problem \eqref{eq:Greedy_Method}, the learner is choosing the action $a \in \mathcal{A}$ that aligns most closely with the Kalman filter prediction $\hat{z}_{t|t-1}$. This implies that the learner is ignoring how the sequence of actions impact the error of the Kalman filter state prediction. In this paper, we will prove how well KODE will perform in comparison to the \textit{Oracle} which is posed in Algorithm \ref{alg:Oracle}. In Algorithm \ref{alg:Oracle}, the \textit{Oracle} chooses actions as follows:
\begin{equation}\label{eq:Oracle}
    a_t^* = \underset{a \in \mathcal{A}}{\arg\max} \left\langle a, z_t \right\rangle. 
\end{equation}

Based on \eqref{eq:Oracle}, the \textit{Oracle} selects the action $a \in \mathcal{A}$ that aligns most closely with the LGDS state variable $z_t$. The \textit{Oracle} assumes knowledge of the LGDS state variable $z_t$, implying it cannot be used as an algorithm for solving Problem \ref{problem_formulation} but is a basis for measuring performance. 

\begin{algorithm}[!t]
\caption{\textit{Oracle}}\label{alg:Oracle}
 \begin{algorithmic}[1]
\STATE \textbf{Input}: $\Gamma$, $\mathcal{A}$, $Q$, $\sigma$, $\Sigma_0$, $z_0$
\FOR{$t=1,2,\dots,n$}
    \STATE $a_t = \underset{a \in \mathcal{A}}{\arg\max} \left\langle a, z_t \right\rangle$  \COMMENT{Select action} 
    \STATE Observe $X_t = \left\langle a_t, z_t \right\rangle + \eta_t$ and $z_t$
\ENDFOR
\end{algorithmic}
\end{algorithm}

\subsection{Performance Analysis}\label{sec:performance}

In the following theorem, we will prove an upper bound for regret $R_n$ defined in \eqref{eq:regret_definition} when using KODE. Next, we will provide a metric that proves when KODE will have a high probability of selecting the same action as the \textit{Oracle}.

\begin{theorem}\label{theorem:Greedy_Regret_Bound}
    Let each action $a \in \mathcal{A}$ be chosen according to optimization problem. Let $P_{\overline{a}}$ be the solution of the following algebraic Riccati equation $P_{\overline{a}} = g\left(P_{\overline{a}},\overline{a}\right)$ and $P_{\overline{a}} \succeq P_a$ for any $a \in \mathcal{A}$. Also assume that $P_{\overline{a}} \succeq P_{0|-1}$. The bound on regret $R_n$ defined as \eqref{eq:regret_definition} is 
    \begin{equation}\label{eq:Greedy_Regret_Bound}
        R_n \leq \max_{a,a' \in \mathcal{A}} n\sqrt{\frac{2\left(a - a'\right)^\top P_{\overline{a}}\left(a - a'\right)}{\pi}}. 
    \end{equation}
\end{theorem}

\begin{proof}
    The instantaneous regret $r_t \triangleq X_t^* - X_t$ for one round $t$ is upper bounded as follows:
    \begin{align}
        r_t &\triangleq X_t^* - X_t \nonumber \\
        & = \left\langle a_t^* , z_t \right\rangle - \left\langle a_t , z_t \right\rangle \nonumber \\
        & \overset{(a)}{=} \left\langle a_t^* , \hat{z}_{t|t-1} + e_{t|t-1} \right\rangle - \left\langle a_t , z_t \right\rangle \nonumber, 
    \end{align}
    \begin{equation}
        \Rightarrow r_t \overset{(b)}{\leq} \left\langle a_t , \hat{z}_{t|t-1} \right\rangle - \left\langle a_t , z_t \right\rangle + \left\langle a_t^* , e_{t|t-1} \right\rangle. \label{eq:greedy_bound}
    \end{equation}

    For $(a)$, we used the expression $z_t = \hat{z}_{t|t-1} + e_{t|t-1}$. As for $(b)$ in \eqref{eq:greedy_bound}, since $a_t \in \mathcal{A}$ is chosen by KODE at round $t$, then 
    \begin{equation}
        \left\langle a_t^* , \hat{z}_{t|t-1} \right\rangle \leq \left\langle a_t , \hat{z}_{t|t-1} \right\rangle \nonumber. 
    \end{equation}

    Therefore, the bound in \eqref{eq:greedy_bound} can be further simplified to 
    \begin{equation}
        r_t \leq - \left\langle a_t , e_{t|t-1}\right\rangle + \left\langle a_t^* , e_{t|t-1} \right\rangle \nonumber. 
    \end{equation}

    Since $\left\langle a_t^* - a_t , e_{t|t-1}\right\rangle  \geq 0 $ and $e_{t|t-1} \sim\mathcal{N}\left(0,P_{t|t-1}\right)$, then $\left\langle a_t^* - a_t , e_{t|t-1}\right\rangle$ is a random variable sampled from the Half-Normal distribution. Therefore, the sum has the expectation 
    \begin{multline}
        \mathbb{E}\left[- \left\langle a_t , e_{t|t-1}\right\rangle + \left\langle a_t^* , e_{t|t-1} \right\rangle\right] \\ = \sqrt{\frac{2\left(a_t^* - a_t\right)^\top P_{t|t-1}\left(a_t^* - a_t\right)}{\pi}} \nonumber, 
    \end{multline}
    implying that expected instantaneous regret has the following upper bound:
    \begin{equation}
        \mathbb{E}\left[r_t\right] \overset{(c)}{\leq} \max_{a,a' \in \mathcal{A}}   \sqrt{\frac{2\left(a-a'\right)^\top P_{\overline{a}}\left(a-a'\right)}{\pi}}, \nonumber. 
    \end{equation}
    
    In $(c)$ we use the fact that $P_{\overline{a}} \succeq P_{0|-1}$ and $P_{\overline{a}} \succeq P_a$ for any $a \in \mathcal{A}$. Therefore, since regret $R_n = \sum_{t=1}^n \mathbb{E}\left[r_t\right]$, regret for KODE has the bound \eqref{eq:Greedy_Regret_Bound}. 
\end{proof}

In Theorem \ref{theorem:Greedy_Regret_Bound}, we derived a linear regret bound for KODE. However, this provides very little intuition on how ``good'' KODE is. To better understand KODE's performance, we will consider how close KODE's action is with respect to the \textit{Oracle}'s action.

The following theorem provides a bound on the angle difference between the LGDS state variable $z_t$ and the Kalman filter state prediction $\hat{z}_{t|t-1}$. This will be used to demonstrate how closely the Kalman filter state prediction $\hat{z}_{t|t-1}$ using KODE aligns with the LGDS state variable $z_t$. This provides a metric for the distance between KODE's action $a_t \in \mathcal{A}$ and the \textit{Oracle}'s action $a_t^* \in \mathcal{A}$. 

\begin{theorem}\label{theorem:equiv_optim}
    Consider the LGDS's state variable $z_t$ and the Kalman state prediction $\hat{z}_{t|t-1}$:
    \begin{equation}\label{eq:computed_angle}
        \theta_t = \arccos\left(\frac{\left\langle z_t, \hat{z}_{t|t-1} \right\rangle}{\left\Vert z_t \right\Vert_2 \left\Vert \hat{z}_{t|t-1}\right\Vert_2}\right). 
    \end{equation}

    The bound on the expected angle $\theta_t\in [0,\pi/4]$ between the state $z_t$ and the Kalman filter state variable prediction $\hat{z}_{t|t-1}$ is 
    \begin{equation}\label{eq:upper_bound}
        \mathbb{E}\left[\theta_t\mid \mathcal{F}_{t-1} \right] \leq \overline{\theta}_t,
    \end{equation}
    where $\overline{\theta}_t$ is defined as
    \begin{equation}\label{eq:theta_bound_def}
        \overline{\theta}_t \triangleq \frac{1}{2}\arccos\left(\frac{2\left\Vert \hat{z}_{t|t-1} \right\Vert_2^2}{\left\Vert \hat{z}_{t|t-1} \right\Vert_2^2 + \mbox{tr}\left(P_{t|t-1}\right)} - 1\right) .
    \end{equation}
    
\end{theorem}

\begin{proof}

    Let $\theta_t\in [0,2\pi)$ be the angle between $z_t$ and $\hat{z}_{t|t-1}$. The expected angle $\mathbb{E}\left[\theta_t\mid \mathcal{F}_{t-1}\right]$ is bounded as follows. First, using the Cauchy-Schwarz inequality and $\hat{z}_{t|t-1} = \mathbb{E}\left[\hat{z}_{t|t-1} \mid \mathcal{F}_{t-1}\right]$, then the following inequalities are true
    \begin{multline}
        \mathbb{E}\left[\left\langle z_t, \hat{z}_{t|t-1}\right\rangle \mid \mathcal{F}_{t-1} \right] = \\
        \mathbb{E}\left[\left\Vert z_t \right\Vert_2 \left\Vert \hat{z}_{t|t-1} \right\Vert_2 \cos\theta_t\mid \mathcal{F}_{t-1} \right] \nonumber, 
    \end{multline}
    \begin{multline}
        \mathbb{E}\left[\left\langle z_t, \hat{z}_{t|t-1}\right\rangle \mid \mathcal{F}_{t-1} \right] \leq \\ \sqrt{\mathbb{E}\left[\left\Vert z_t \right\Vert_2^2 \mid \mathcal{F}_{t-1} \right]\mathbb{E}\left[\left\Vert \hat{z}_{t|t-1} \right\Vert_2^2 \mid \mathcal{F}_{t-1} \right]} \cdot \\
        \sqrt{\mathbb{E}\left[\cos^2\theta_t\mid \mathcal{F}_{t-1} \right]} \nonumber, 
    \end{multline}
    \begin{multline}
        \Rightarrow \mathbb{E}\left[\left\langle z_t, \hat{z}_{t|t-1}\right\rangle \mid \mathcal{F}_{t-1} \right] \leq \\ \sqrt{\left\Vert \hat{z}_{t|t-1} \right\Vert_2^2\mathbb{E}\left[\left\Vert z_t \right\Vert_2^2 \mid \mathcal{F}_{t-1} \right]\mathbb{E}\left[\cos^2\theta_t\mid \mathcal{F}_{t-1} \right]} \nonumber.
    \end{multline}

    For the next step, the state $z_t$ can be expressed $z_t = \hat{z}_{t|t-1} + e_{t|t-1}$. In addition, according to Lemma \ref{lemma:Kalman_Facts}, $\mathbb{E}\left[\left\langle \hat{z}_{t|t-1}, e_{t|t-1}\right\rangle\right] = 0$. Therefore, 
    \begin{multline}
        \frac{\mathbb{E}\left[\left\langle \hat{z}_{t|t-1} + e_{t|t-1}, \hat{z}_{t|t-1}\right\rangle \mid \mathcal{F}_{t-1} \right]^2}{\left\Vert \hat{z}_{t|t-1} \right\Vert_2^2\mathbb{E}\left[\left\Vert \hat{z}_{t|t-1} + e_{t|t-1} \right\Vert_2^2 \mid \mathcal{F}_{t-1} \right]} \leq \\ \mathbb{E}\left[\cos^2 \theta_t\mid \mathcal{F}_{t-1} \right] \nonumber, 
    \end{multline}
    \begin{multline}
        \frac{\mathbb{E}\left[\left\Vert \hat{z}_{t|t-1} \right\Vert_2^2 \mid \mathcal{F}_{t-1} \right]^2}{\left\Vert \hat{z}_{t|t-1} \right\Vert_2^2\mathbb{E}\left[\left\Vert \hat{z}_{t|t-1} \right\Vert_2^2 + \left\Vert e_{t|t-1} \right\Vert_2^2 \mid \mathcal{F}_{t-1} \right]} \leq \\ \mathbb{E}\left[\cos^2 \theta_t\mid \mathcal{F}_{t-1} \right] \nonumber , 
    \end{multline}
    \begin{multline}
        \frac{\left\Vert \hat{z}_{t|t-1} \right\Vert_2^2}{\left\Vert \hat{z}_{t|t-1} \right\Vert_2^2 + \mathbb{E}\left[\left\Vert e_{t|t-1} \right\Vert_2^2\mid \mathcal{F}_{t-1} \right]} \leq \\ \mathbb{E}\left[\frac{1+\cos 2\theta_t}{2} \mid \mathcal{F}_{t-1} \right] \nonumber, 
    \end{multline}
    \begin{equation}\label{eq:term_1}
        \Rightarrow \frac{\left\Vert \hat{z}_{t|t-1} \right\Vert_2^2}{\left\Vert \hat{z}_{t|t-1} \right\Vert_2^2 + \mbox{tr}\left(P_{t|t-1}\right)} \leq \mathbb{E}\left[\frac{1+\cos 2\theta_t}{2} \mid \mathcal{F}_{t-1} \right], 
    \end{equation}

    The function $\cos2\theta_t$ is concave in the interval $\theta_t \in \left[0,\pi/4\right]$. Using Jensen's inequality \cite{boucheron2013concentration}, the right side of \eqref{eq:term_1} is upper-bounded as 
    \begin{equation}
        \frac{2\left\Vert \hat{z}_{t|t-1} \right\Vert_2^2}{\left\Vert \hat{z}_{t|t-1} \right\Vert_2^2 + \mbox{tr}\left(P_{t|t-1}\right)} -1 \leq \cos 2\mathbb{E}\left[\theta_t\mid \mathcal{F}_{t-1} \right] \nonumber, 
    \end{equation}
    \begin{equation}
        \Rightarrow \frac{1}{2}\arccos\left(\frac{2\left\Vert \hat{z}_{t|t-1} \right\Vert_2^2}{\left\Vert \hat{z}_{t|t-1} \right\Vert_2^2 + \mbox{tr}\left(P_{t|t-1}\right)} - 1\right) \overset{(a)}{\geq} \mathbb{E}\left[\theta_t\mid \mathcal{F}_{t-1} \right] \nonumber ,
    \end{equation}
    where in $(a)$ we used the fact that the inverse of $\cos\left(\cdot\right)$ switches the inequality. Therefore, the bound for $\mathbb{E}\left[\theta_t\mid \mathcal{F}_{t-1} \right]$ is \eqref{eq:upper_bound}. If the bound is greater than $\pi/4$, then it is no longer viable.     
\end{proof}

Theorem \ref{theorem:equiv_optim} provides a bound on the angle between the LGDS's state variable $z_t$ and the Kalman filter state prediction $\hat{z}_{t|t-1}$. By bounding this angle, we can observe if KODE and the \textit{Oracle} will choose the same action $a_t = a_t^*$. This metric is more informative since we are now measuring how close the two actions $a_t,a_t^*$ are to each other. 

A major issue in Theorem \ref{theorem:equiv_optim} are the terms $\left\Vert \hat{z}_{t|t-1}\right\Vert_2$ and $\mbox{tr}\left(P_{t|t-1}\right)$. Since we can only observe these values at time $t$, the bound \eqref{eq:upper_bound} can only be computed online. In addition, this gives no intuition on how the properties of the LGDS \eqref{eq:known_LGDS} impact the performance of KODE. Therefore, in the next theorem, we will provide an upper bound that can be computed without observing $\left\Vert \hat{z}_{t|t-1}\right\Vert_2$ and $\mbox{tr}\left(P_{t|t-1}\right)$. This will provide a perspective on how the properties of the LGDS impact the performance of KODE. 

\begin{theorem}\label{theorem:sensitivity_condition}
    Assume that $\left(\Gamma, Q^{1/2}\right)$ is controllable and $\left(\Gamma, a\right)$ is detectable for any $a \in \mathcal{A}$. Let $P_a$ be the steady state solution of the Kalman filter for each action $a \in \mathcal{A}$, 
    \begin{equation}
        P_a = g\left(P_a,a\right) \nonumber,
    \end{equation}
    where $g\left(P_a,a\right)$ is defined to be \eqref{eq:riccati_equation}. Let there be $P_{\overline{a}}$ where for each $a \in \mathcal{A}$, $P_{\overline{a}} \succeq P_a$. Define $Z_t = \mathbb{E}\left[z_tz_t^\top\right]$ where $Z = \lim_{t\rightarrow \infty} Z_t$ is the solution of the Lyapunov equation $Z = \Gamma Z\Gamma^\top + Q$. Assuming that $P_{\overline{a}}\succeq P_{0|-1}$ and $Z \succeq Z_0$, we have the inequality
    \begin{equation}\label{eq:sensitivity_condition}
        \frac{1}{2}\arccos\left(\frac{2\left\Vert \hat{z}_{t|t-1} \right\Vert_2^2}{\left\Vert \hat{z}_{t|t-1} \right\Vert_2^2 + \mbox{tr}\left(P_{t|t-1}\right)} - 1\right) \leq \overline{\theta}_S, 
    \end{equation}
    \begin{equation}\label{eq:theta_bound_steady_def}
        \overline{\theta}_S \triangleq \frac{1}{2}\arccos\left(\frac{2\nu}{\nu + \mbox{tr}\left(P_{\overline{a}}\right)} - 1\right). 
    \end{equation}
    
    The variable $\nu$ in \eqref{eq:theta_bound_steady_def} is a threshold value such that with a probability $\alpha$ the following inequality is true: 
    \begin{equation}
        w_t^\top \left(Z-P_{\overline{a}}\right) w_t \geq \nu, ~ w_t \sim \mathcal{N}(0,I_d) \nonumber.
    \end{equation}
\end{theorem}
\begin{proof}

    The function $g\left(P_a,a\right)$ is monotonic increasing according to \cite{1333199}. Using proof by induction, for the base case, we first have the following inequality which is satisfied by monotonicity of $g\left(P_a,a\right)$ with respect to $P_a$:
    \begin{equation}
        P_{\overline{a}} \succeq P_{0|-1} \Rightarrow g\left(P_{\overline{a}},a\right) \succeq g\left(P_{0|-1},a\right)\nonumber,
    \end{equation}
    \begin{equation}
        \Rightarrow P_{\overline{a}} \succeq g\left(P_{0|-1},a\right) \mbox{ since } P_{\overline{a}} \succeq g\left(P_{\overline{a}},a\right)
        \nonumber. 
    \end{equation}

    For the induction step for proof by induction, we have
    \begin{equation}
        P_{\overline{a}} \succeq P_{t|t-1} \Rightarrow g\left(P_{\overline{a}},a\right) \succeq g\left(P_{t|t-1},a\right)\nonumber, 
    \end{equation}
    \begin{equation}
        \Rightarrow P_{\overline{a}} \succeq g\left(P_{t|t-1},a\right) \mbox{ since } P_{\overline{a}} \succeq g\left(P_{\overline{a}},a\right) \label{eq:matrix_inequality2}. 
    \end{equation}

    The above implies that if $P_{\overline{a}} \succeq P_{0|-1}$, then $P_{\overline{a}} \succeq P_{t|t-1}$ always. Therefore, we have the following inequality
    \begin{equation}
        \mbox{tr}\left(P_{\overline{a}}\right) \geq \mbox{tr}\left(P_{t|t-1}\right) \nonumber.
    \end{equation}

    For the next part, by orthogonality principle, $\mathbb{E}\left[\hat{z}_{t|t-1} e_{t|t-1}^\top \mid \mathcal{F}_{t-1}\right] = \mathbb{E}\left[e_{t|t-1} \hat{z}_{t|t-1}^\top \mid \mathcal{F}_{t-1}\right] = \mathbf{0}$. Therefore, since $z_t = \hat{z}_{t|t-1} + e_{t|t-1}$, we can express $Z_t = \mathbb{E}\left[z_t z_t^\top \right]$ as follows: 
    \begin{align}
        Z_t & = \mathbb{E}\left[z_t z_t^\top \right] \nonumber\\
        & = \mathbb{E}\left[(\hat{z}_{t|t-1} + e_{t|t-1}) (\hat{z}_{t|t-1} + e_{t|t-1})^\top\right] \nonumber\\
        & = \mathbb{E}\left[\hat{z}_{t|t-1}\hat{z}_{t|t-1}^\top \right] + \mathbb{E}\left[\hat{z}_{t|t-1} e_{t|t-1}^\top \right]  + \mathbb{E}\left[e_{t|t-1}\hat{z}_{t|t-1}^\top \right] \nonumber \\
        & ~~+ \mathbb{E}\left[e_{t|t-1} e_{t|t-1}^\top \right] \nonumber\\
        & \overset{(a)}{=} \hat{Z}_{t|t-1} + P_{t|t-1} \nonumber,
    \end{align}
    \begin{equation}
        \Rightarrow \hat{Z}_{t|t-1} = Z_t - P_{t|t-1} \label{eq:Kalman_Filter_orthogonality_principle}, 
    \end{equation}
    where in $(a)$ we used Lemma \ref{lemma:Kalman_Facts} and 
    \begin{multline}
        \mathbb{E}\left[\hat{z}_{t|t-1} e_{t|t-1}^\top \right] = \mathbb{E}\left[\mathbb{E}\left[\hat{z}_{t|t-1} e_{t|t-1}^\top \mid \mathcal{F}_{t-1} \right]\right] = \mathbf{0}, 
    \end{multline}
    
    Let $Z$ be the solution of the Lyapunov equation $Z = \Gamma Z \Gamma^\top +Q$. Since $P_{\overline{a}} \succeq P_{t|t-1}$ is satisfied and $Z_t \rightarrow Z$, then 
    \begin{equation}
        \hat{Z}_{t|t-1} = Z - P_{t|t-1} \nonumber, 
    \end{equation}
    \begin{equation}
        \Rightarrow \hat{Z}_{t|t-1} \succeq Z- P_{\overline{a}} \label{eq:relationship_z_hat}. 
    \end{equation}
    
    Therefore, the variance of $\mathbb{E}\left[\left\Vert\hat{z}_{t|t-1}\right\Vert_2^2 \right] = \mbox{tr}\left(\hat{Z}_{t|t-1}\right) \geq \mbox{tr}\left(Z-P_{\overline{a}}\right)$. Let there be two random variables 
    \begin{align}
        v_t^\top v_t & =  w_t^\top \left(Z-P_{\overline{a}}\right) w_t  \nonumber\\
        \hat{z}_{t|t-1}^\top \hat{z}_{t|t-1} & = w_t^\top \hat{Z}_{t|t-1} w_t \nonumber\\
        w_t & \sim \mathcal{N}(0,I_d) \nonumber.
    \end{align}

    Equation \eqref{eq:relationship_z_hat} demonstrates that $v_t^\top v_t \leq \hat{z}_{t|t-1}^\top \hat{z}_{t|t-1}$. With a probability of $\alpha$, $\alpha \in [0,1]$ where $v_t^\top v_t \geq \nu$, 
    \begin{equation}
        \hat{z}_{t|t-1}^\top \hat{z}_{t|t-1} \geq v_t^\top v_t \geq \nu \nonumber.
    \end{equation}

    The bound \eqref{eq:upper_bound} is convex with respect to $\left\Vert \hat{z}_{t|t-1} \right\Vert_2$ and concave with respect to $\mbox{tr}\left(P_{t|t-1}\right)$. Using the lower and upper bounds of $\left\Vert \hat{z}_{t|t-1} \right\Vert_2$ and $\mbox{tr}\left(P_{t|t-1}\right)$ respectively, \eqref{eq:sensitivity_condition} is satisfied with a probability $\alpha$.
\end{proof}

Theorem \ref{theorem:sensitivity_condition} states that as the trace $\mbox{tr}\left(P_{\overline{a}}\right)$ of the Kalman filter error covariance matrix decreases, then the angle between the LGDS's state variable $z_t$ and Kalman filter state prediction $\hat{z}_{t|t-1}$ decreases. Therefore, the trace of the Kalman filter error covariance matrix $\mbox{tr}\left(P_{\overline{a}}\right)$ impacts the probability that KODE chooses the same action as the \textit{Oracle}, i.e. $a_t^* = a_t$. In the next subsection, we will connect the observability of the LGDS \eqref{eq:known_LGDS} to KODE. This will provide some intuition on why KODE may perform well for certain LGDS's.

\subsection{LGDS Observability Impact on KODE's Exploration}\label{subsec:greedy_performance}

Based on the results of Theorem \ref{theorem:sensitivity_condition}, as the error covariance matrix $\mbox{tr}\left(P_{\overline{a}}\right)$ decreases, the Kalman filter state prediction $\hat{z}_{t|t-1}$ becomes more aligned with LGDS's state variable $z_t$. In this subsection, we will explain that, despite KODE being an exploration-free method, the LGDS will perturb KODE to explore each observable action $a \in \mathcal{A}$. First, the Kalman filter \eqref{eq:Kalman_Filter} state prediction $\hat{z}_{t+1|t}$ can be expressed the following way
\begin{equation}
    \hat{z}_{t+1|t} = \Gamma \hat{z}_{t|t-1} + \Gamma P_{t|t-1} \frac{a_t}{\sigma^2}\sqrt{a_t^\top P_{t|t-1}a_t + \sigma^2}\omega_t \nonumber,
\end{equation}
where $\omega_t \in \mathbb{R}$ is white noise, i.e. $\omega_t \sim \mathcal{N}\left(0,1\right)$. Therefore, the reward prediction for each action $a \in \mathcal{A}$ is expressed as 
\begin{multline}
    \left\langle a, \hat{z}_{t+1|t} \right\rangle = \left\langle a, \Gamma \hat{z}_{t|t-1} \right\rangle \\+ \left\langle a, \Gamma P_{t|t-1} \frac{a_t}{\sigma^2}\sqrt{a_t^\top P_{t|t-1}a_t + \sigma^2}\omega_t\right\rangle \nonumber,
\end{multline}
\begin{multline}\label{eq:Greedy_Method_Expression}
    \Rightarrow \left\langle a, \hat{z}_{t+1|t} \right\rangle = \left\langle a, \hat{z}_{t+1|t-1} \right\rangle \\ + \left\langle a, \Gamma P_{t|t-1} \frac{a_t}{\sigma^2}\sqrt{a_t^\top P_{t|t-1}a_t + \sigma^2}\omega_t\right\rangle .
\end{multline}

In \eqref{eq:Greedy_Method_Expression}, there appears to be a zero-mean Gaussian random variable that is perturbing $\left\langle a, \hat{z}_{t+1|t-1}\right\rangle$ which will define as $u_t\left(a\mid a_t\right)$:
\begin{equation}\label{eq:implicit_exploration}
    u_t\left(a\mid a_t\right) \triangleq \left\langle a, \Gamma P_{t|t-1} \frac{a_t}{\sigma^2}\sqrt{a_t^\top P_{t|t-1}a_t + \sigma^2}\omega_t\right\rangle. 
\end{equation}

If the random variable $u_t\left(a\mid a_t\right)$ defined as \eqref{eq:implicit_exploration} is perturbing the learner to select actions $a \neq a_t$, then this implies that there is exploration occurring when using KODE. The following theorem analyzes the behavior of the random variable $u_t\left(a\mid a_t\right)$ in \eqref{eq:implicit_exploration}, which we will denote as an implicit exploration term.


\begin{theorem}\label{theorem:divergence_undetectable}
    Let there be the LGDS \eqref{eq:known_LGDS}. Let the pair $\left(\Gamma,a\right)$, $a \in \mathcal{A}$, be unobservable and action $\tilde{a} \in \mathcal{A}$ observe the states unobserved by action $a \in \mathcal{A}$. Finally, let the LGDS \eqref{eq:known_LGDS} use a similarity transformation matrix $T \in \mathbb{R}^{d \times d}$ such that the state $z_t$ and actions $a,\tilde{a} \in \mathcal{A}$ are decomposed into 
    \begin{align}
        T\begin{pmatrix}
            z_t^O \\ 
            z_t^U 
        \end{pmatrix} = z_t, ~ T^\top a & = \begin{pmatrix}
            a_O \\
            \mathbf{0}
        \end{pmatrix}, ~ T^\top \tilde{a} = \begin{pmatrix}
            \tilde{a}_O \\
            \tilde{a}_U
        \end{pmatrix} \nonumber, 
    \end{align}
    
    Therefore, the LGDS \eqref{eq:known_LGDS} can be expressed as: 
    \begin{equation}\label{eq:converted_LGDS}
        \begin{cases}
            \begin{pmatrix}
                z_{t+1}^O \\
                z_{t+1}^U
            \end{pmatrix} & = \begin{pmatrix}
                \Gamma_O & \mathbf{0} \\
                \Gamma_{U'} & \Gamma_U
            \end{pmatrix}\begin{pmatrix}
                z_t^O \\
                z_t^U
            \end{pmatrix} + \xi_t'\\
            X_t & = \left\langle \begin{pmatrix}
                a_O \\
                \mathbf{0}
            \end{pmatrix}, \begin{pmatrix}
                z_t^O \\
                z_t^U
            \end{pmatrix}\right\rangle + \eta_t
        \end{cases},
    \end{equation}
    \begin{equation}
        \begin{pmatrix}
                z_0^O \\
                z_0^U
            \end{pmatrix} \sim \mathcal{N}\left(\begin{pmatrix}
                \mathbf{0} \\
                \mathbf{0}
            \end{pmatrix}, \begin{pmatrix}
                P_{0|-1}^O & \mathbf{0}\\
                \mathbf{0} & P_{0|-1}^U
            \end{pmatrix}\right) \nonumber,
    \end{equation}
    where $\left(\Gamma_O,a_O^\top \right)$ is observable. Let the error covariance matrix $P_{t|t-1}$ of the Kalman filter \eqref{eq:Kalman_Filter} be converted as follows
    \begin{equation}\label{eq:converted_kalman_matrix}
        P_{t|t-1} \rightarrow T^{-1} P_{t|t-1} T^{-\top} = \begin{pmatrix}
            P_{t|t-1}^O & \Phi_{t|t-1} \\
            \Phi_{t|t-1}^\top & P_{t|t-1}^U
        \end{pmatrix},
    \end{equation}
    
    An action $\Tilde{a} \in \mathcal{A}$ has an implicit exploration term $u_t\left(\tilde{a}\mid a\right)$ \eqref{eq:implicit_exploration} that is almost surely zero---$P\left(u_t\left(\tilde{a}\mid a\right) = 0\right) = 1$---if all of the following conditions apply:
    \begin{itemize}
        \item The state matrix and action pair $\left(\Gamma,\Tilde{a}\right)$ is unobservable, i.e. $\Gamma_U' = \mathbf{0}$ and $\Tilde{a}_O = \mathbf{0}$. 
        \item The correlation of the state prediction error between the unobserved states $z_t^U$ and observed states $z_t^O$ is zero, i.e. $\Phi_{t|t-1} = \mathbf{0}$. 
    \end{itemize}
\end{theorem}
\begin{proof}
    Let action $a \in \mathcal{A}$ be an unobservable action, i.e. the pair $\left(\Gamma,a^\top \right)$ is unobservable. We can decompose the LGDS \eqref{eq:known_LGDS} using a similarity transformation matrix $T \in \mathbb{R}^{d \times d}$ into \eqref{eq:converted_LGDS}. Let the error covariance matrix $P_{t|t-1}$ of the Kalman filter be converted to \eqref{eq:converted_kalman_matrix}. Assume that action $a \in \mathcal{A}$ has been chosen. The term $u_t\left(a \mid a\right)$ in \eqref{eq:implicit_exploration} for action $a \in \mathcal{A}$ can be expressed as 
    \begin{multline}
        \Gamma P_{t|t-1}= \\ \begin{pmatrix}
            \Gamma_O P_{t|t-1}^O & \Gamma_O \Phi_{t|t-1} \\
            \Gamma_{U'} P_{t|t-1}^O + \Gamma_U \Phi_{t|t-1}^\top & \Gamma_{U'}\Phi_{t|t-1} + \Gamma_U P_{t|t-1}^U
        \end{pmatrix} \nonumber, 
    \end{multline}
    \begin{equation}
        \Rightarrow u_t\left(a \mid a\right) = \frac{a_O^\top \Gamma_O P_{t|t-1}^O a_O}{\sqrt{a_O^\top P_{t|t-1}^O a_O + \sigma^2}}\omega_t. \nonumber
    \end{equation}
    
    For action $\Tilde{a} \in \mathcal{A}$, the term $u_t\left(\tilde{a}\mid a\right)$ can be expressed as 
    \begin{multline}\label{eq:unobserved_action}
        \Rightarrow u_t\left(\tilde{a}\mid a\right) = \frac{\Tilde{a}_O^\top \Gamma_O P_{t|t-1}^O a_O}{\sqrt{a_O^\top P_{t|t-1}^O a_O + \sigma^2}}\omega_t \\ + \frac{\Tilde{a}_U^\top \Gamma_{U'}P_{t|t-1}^O a_O}{\sqrt{a_O^\top P_{t|t-1}^O a_O + \sigma^2}}\omega_t + \frac{\Tilde{a}_U^\top \Gamma_U\Phi_{t|t-1} a_O}{\sqrt{a_O^\top P_{t|t-1}^O a_O + \sigma^2}}\omega_t. 
    \end{multline}
    
    The following can be inferred from expression \eqref{eq:unobserved_action}: 
    \begin{itemize}
        \item If $\left(\Gamma, \Tilde{a}^\top \right)$ is unobservable and observes the states $z_t^U$ unobserved by $a$, then both $\Gamma_{U'} = \mathbf{0}$ and $\Tilde{a}_O = \mathbf{0}$. Therefore expression \eqref{eq:unobserved_action} for $u_t\left(\tilde{a}\mid a\right)$ can be simplified to 
        \begin{equation}\label{eq:correlated_error}
            u_t\left(\tilde{a}\mid a\right) = \frac{\Tilde{a}_U^\top \Gamma_U\Phi_{t|t-1} a_O}{\sqrt{a_O^\top P_{t|t-1}^O a_O + \sigma^2}}\omega_t . 
        \end{equation}
        
        Expression \eqref{eq:correlated_error} implies that if $\Tilde{a}_U^\top \Gamma_U\Phi_{t|t-1} a_O \neq 0$, then $u_t\left(\tilde{a}\mid a\right)$ defined as \eqref{eq:implicit_exploration} is almost surely nonzero for action $\Tilde{a} \in \mathcal{A}$. 
        \item If $\left(\Gamma, \Tilde{a}^\top \right)$ is observable, then at most one of the values $\Gamma_{U'}$ or $\tilde{a}_O$ can be set to $\mathbf{0}$. Therefore, if the correlation between state prediction errors $\Phi_{t|t-1}$ for state $z_t^U$ and $z_t^O$ is $\mathbf{0}$, expression \eqref{eq:unobserved_action} for $u_t\left(\tilde{a}\mid a\right)$ can be simplified to any of the following
        \begin{multline}
           u_t\left(\tilde{a}\mid a\right) = \\ \begin{cases}
                \frac{\Tilde{a}_O^\top \Gamma_O P_{t|t-1}^O a_O}{\sqrt{a_O^\top P_{t|t-1}^O a_O + \sigma^2}}\omega_t  & \mbox{ if } \Gamma_{U'} = \mathbf{0} \mbox{ and } \Tilde{a}_O \neq \mathbf{0}\\
                \frac{\Tilde{a}_U^\top \Gamma_{U'}P_{t|t-1}^O a_O}{\sqrt{a_O^\top P_{t|t-1}^O a_O + \sigma^2}} \omega_t  & \mbox{ if } \Gamma_{U'} \neq \mathbf{0} \mbox{ and } \Tilde{a}_O = \mathbf{0} \\
            \end{cases} \nonumber.  
        \end{multline}

        Therefore, since either $\Tilde{a}_O^\top \Gamma_O P_{t|t-1}^O a_O \neq 0$ or $\Tilde{a}_U^\top \Gamma_{U'}P_{t|t-1}^O a_O \neq 0$, then the term $u_t\left(\tilde{a}\mid a\right)$ is almost surely nonzero. 
    \end{itemize}

    Based on the points above, the term $u_t\left(\tilde{a}\mid a\right)$ defined as  \eqref{eq:implicit_exploration} is almost surely zero for action $\tilde{a} \in \mathcal{A}$---given that action $a \in \mathcal{A}$ has been previously chosen---if either  $\left(\Gamma, \tilde{a}^\top \right)$ is unobservable and $\Phi_{t|t-1} = \mathbf{0}$ (no correlation between the errors of the state predictions $\hat{z}_{t|t-1}^O$ and $\hat{z}_{t|t-1}^U$). 
\end{proof}

Theorem \ref{theorem:divergence_undetectable} proves that $u_t\left(\tilde{a}\mid a\right)$ defined as \eqref{eq:implicit_exploration} consistently perturbs KODE to explore observable actions $\tilde{a} \in \mathcal{A}$ unless $\tilde{a}^\top \Gamma P_{t|t-1} a = 0$ for some $P_{t|t-1}$. Consequently, action $\tilde{a} \in \mathcal{A}$ is perturbed by $u_t\left(\tilde{a}\mid a\right)$ if either condition is true:
\begin{enumerate}
    \item \textbf{Condition 1}: There exists a subspace $z_t^O$ (see \eqref{eq:converted_LGDS}) of the LGDS state variable $z_t$ that is observed by $\tilde{a}$ and $a$. 
    \item \textbf{Condition 2}: The reward prediction errors of actions $\tilde{a}$ and $a$ are correlated,    
    \begin{equation}
        \mathbb{E}\left[\left(X_t - \left\langle a_t, \hat{z}_{t|t-1} \right\rangle \right)\left(X_t - \left\langle \tilde{a}_t, \hat{z}_{t|t-1} \right\rangle \right)^\top \right] \neq 0 \nonumber, 
    \end{equation}
    where $X_t$ and $\tilde{X}_t$ are expressed as 
    \begin{align}
        X_t & = \left\langle a, z_t \right\rangle + \eta_t \nonumber\\
        \tilde{X}_t & = \left\langle \tilde{a}, z_t \right\rangle + \eta_t \nonumber. 
    \end{align}

\end{enumerate}



\section{Numerical Results}\label{sec:numerical_results}


We evaluate the performance of KODE and compare it to several benchmark methods: UCB \cite{agrawal1995sample}, Sliding Window UCB (SW-UCB) \cite{garivier2008upper}, Rexp3 \cite{besbes2014stochastic}, OFUL \cite{NIPS2011_e1d5be1c}, and a method that randomly selects actions denoted as Random. We include UCB because the LGDS is stable which implies that the output rewards are stationary. SW-UCB is added to since it can adjust to the time correlations of the reward. As for Rexp3, the algorithm has addressed environments similarly to the one proposed in \eqref{eq:known_LGDS}. OFUL is a state-of-the-art linear bandit method, making it a relevant benchmark to our proposed environment.

We generated $10^3$ distinct LGDS instances where the dimension of $z_t$ is $d = 10$ and the number of actions $k = 10$. For each LGDS, the matrices $Q^{1/2}$, $\sigma$, $a \in \mathcal{A}$ are independently sampled from a Gaussian distribution $\mathcal{N}\left(0,1\right)$, with each component of the vectors or matrices independently sampled. To constrain the eigenvalues of the state matrix $\Gamma \in \mathbb{R}^{d \times d}$ in the unit circle, we generated a matrix $G$ from a Gaussian distribution $\mathcal{N}\left(0,1\right)$ and set $\Gamma = \left(0.99/\lambda_{\max}\left(G\right)\right)G$.


In this simulation study, $10$ separate simulations were computed for each randomly generated LGDS. The horizon length $n$ was set to $n = 10^3$. Before each algorithm's interaction with the LGDS, the LGDS was initialized by computing $10^4$ iterations to reach a steady-state distribution. 

Figure \ref{figure:regret_comparison} shows a boxplot illustrating the decrease in regret using KODE compared to the other algorithms. The center of the box is the median with bottom and top boxes edges as the first and third quantiles, respectively. The positive median and quantile percentages suggest that KODE consistently outperforms the comparison methods as depicted in Figure \ref{figure:regret_comparison}. This demonstrates that an exploration-free method is able to outperform a number of well-known SMAB algorithms. In addition, the results suggest that KODE is capable of finding actions that output higher rewards in comparison to the other methods, strengthening our findings that the term $u_t\left(a \mid a_t\right)$ defined as \eqref{eq:implicit_exploration} is perturbing unexplored actions $a \in \mathcal{A}$.

\begin{figure}[t]
    \centering
    \includegraphics[width=\linewidth]{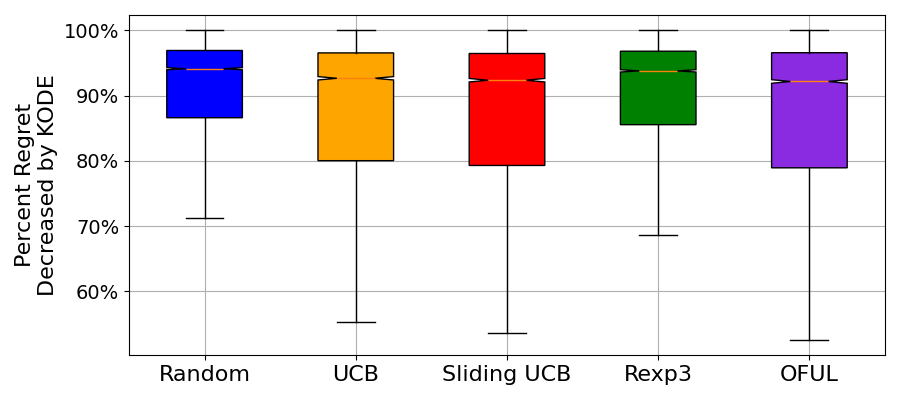}
    \caption{Comparison of KODE with various SMAB algorithms. The percentage indicates the reduction in regret achieved by KODE relative to the compared algorithm.}
    \label{figure:regret_comparison}
\end{figure}

\subsection{Studying Implicit Exploration}

In this analysis, we aim to connect the theoretical developments of $u_t\left(\tilde{a}\mid a\right)$, defined in \eqref{eq:implicit_exploration} and analyzed in Theorem \ref{theorem:divergence_undetectable}, to KODE's empirical relative performance in Figure \ref{figure:regret_comparison}. According to Theorem \ref{theorem:divergence_undetectable}, if two actions $a,a_t \in \mathcal{A}$, where $a \neq a_t$, fail both \textbf{Condition 1} and \textbf{Condition 2} in Subsection \ref{subsec:greedy_performance}, the $u_t\left(a \mid a_t\right)$ is almost surely zero. To understand the impact of the term $u_t\left(a \mid a_t\right)$ on KODE's performance in terms of regret, we will analyze the correlation between KODE's percent regret decreased in Figure \ref{figure:regret_comparison} and the maximum obtainable variance of the term $u_t\left(a \mid a_t\right)$ \eqref{eq:implicit_exploration}, denoted as $\tilde{u}$. The term $\tilde{u}$ is defined as
\begin{equation}\label{eq:tilde_u_def}
    \tilde{u} \triangleq \max_{a,\tilde{a} \in \mathcal{A}, a \neq \tilde{a}} \mathbb{E}\left[\left\langle a, \Gamma P_{\overline{a}} \frac{\tilde{a}}{\sigma^2}\sqrt{a_t^\top P_{\overline{a}}\tilde{a} + \sigma^2}\omega_t\right\rangle^2 \right], 
\end{equation}
\begin{equation}
    \Rightarrow \tilde{u} =  \max_{a,\tilde{a} \in \mathcal{A}, a \neq \tilde{a}}  \left(\frac{a^\top \Gamma P_{\overline{a}} \tilde{a}  \tilde{a}^\top P_{\overline{a}} \Gamma^\top a}{\tilde{a}^\top P_{\overline{a}}\tilde{a} + \sigma^2}\right) \nonumber.  
\end{equation}

Recall that $P_{\overline{a}}$ is defined such that $P_{\overline{a}} \succeq P_a$, where $P_a$ solves the algebraic Riccati equation $P_a = g\left(P_a,a\right)$ for each $a \in \mathcal{A}$. The variable $\tilde{u}$ is a metric of observability that incorporates the noise statistics $Q$ and $\sigma^2$ which are not considered in the \textit{Observability Gramian} defined in \eqref{eq:observability_gramian}. According to Theorem \ref{theorem:divergence_undetectable}, the maximum variance of the term $u_t\left(a \mid a_t\right)$ defined as $\tilde{u}$ in \eqref{eq:tilde_u_def} is zero if both \textbf{Condition 1} and \textbf{Condition 2} are not satisfied for an action pair $\left(a,a_t\right)$. With zero variance $\tilde{u}$, no perturbation for exploring action $a \in \mathcal{A}$ is added. Therefore, if action $a$ outputs the highest reward, then KODE's percent regret decreased is expected to be lower. This simulation study's objective is to test our null hypothesis: there does not exist a positive correlation between the magnitude of $\tilde{u}$ and the KODE's percent regret decreased with respect to comparison methods. 

In Figure \ref{figure:correlation_values}, we plot the Pearson $r$ correlation between KODE's percent regret decreased in Figure \ref{figure:regret_comparison} and $\log_{10}\tilde{u}$ where $\tilde{u}$ as defined in \eqref{eq:tilde_u_def}. We use $\log_{10} \tilde{u}$ since the range of $\tilde{u}$ is from $10^{-2}$ to $10^4$. Each bar is above $0.4$ correlation with $p$-values less than $10^{-2}$. Our null hypothesis states that there exists no correlation between KODE's percent regret decreased and $\tilde{u}$. However, with Pearson $r$ correlation values of at least $0.4$ and significant $p$-values, this suggests moderate statistical correlation between $\tilde{u}$ and KODE's percent regret decreased. 

The key takeaway with these results is that as the LGDS becomes less observable, the worse KODE performs (in terms of percent regret decreased). According to the Pearson $r$ correlations in Figure \ref{figure:correlation_values}, a smaller $\tilde{u}$ implies a decrease in KODE's percent regret decreased. This supports the findings of Theorem \ref{theorem:divergence_undetectable} which suggested that as the LGDS becomes less observable, less exploration for actions $a \in \mathcal{A}$ occurs. In consequence, KODE performs worse. 

\begin{figure}[t]
    \centering
    \includegraphics[width=\linewidth]{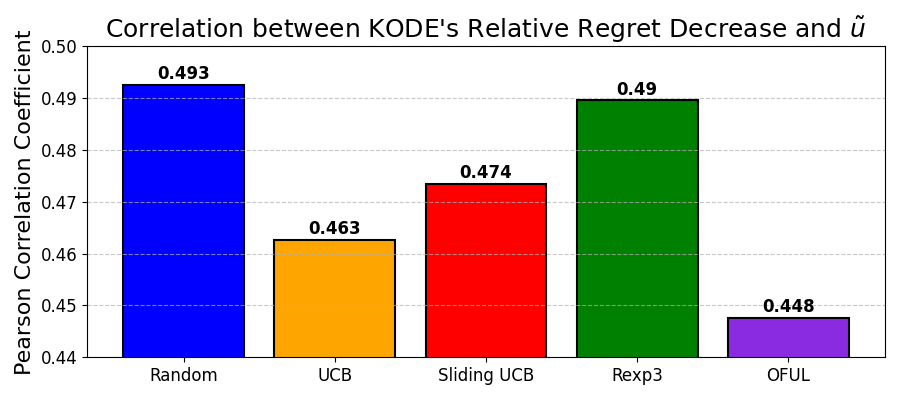}
    \caption{Comparison of KODE with various SMAB algorithms. The percentage indicates the reduction in regret achieved by KODE relative to the compared algorithm.}
    \label{figure:correlation_values}
\end{figure}

\section{Conclusion}\label{sec:Conclusion}

We presented a Stochastic Multi-Armed Bandit (SMAB) problem within a known Linear Gaussian Dynamical System (LGDS) environment. The goal was to maximize the cumulative reward over a horizon $n$, where the reward was the output of a LGDS. This works core contribution is the analysis of our proposed method, the Kalman filter Observability Dependent Exploration (KODE) algorithm---an exploration-free method---where the learner selects the action vector that aligns most closely with the Kalman filter state prediction, the optimal one-step state predictor of a LGDS. We provided bounds on the performance of KODE with respect to the \textit{Oracle}'s performance, where the bounds are dependent on the LGDS parameters. In our analysis, we discovered an implicit exploration term that promotes exploration in KODE depending on the observability of the LGDS. Finally, we validated the results through a simple numerical example.

\bibliographystyle{IEEEtran}
\bibliography{IEEEabrv,autosam}{}

\end{document}